\documentclass[nohyperref]{article}

\usepackage{microtype}
\usepackage{graphicx}
\usepackage{subfigure}
\usepackage{booktabs} 
\usepackage{amsmath}
\usepackage{thmtools}
\usepackage{thm-restate}
\usepackage{Definitions}
\DeclareMathOperator{\sg}{sg}
\newcommand{\hht}{\hat t}

\newcommand{\bCcal}{{\bar \Ccal}}

\newcommand{\vkappa}{\kappa}

\usepackage{hyperref}

\usepackage[accepted]{icml2022}

\usepackage{amsmath}
\usepackage{amssymb}
\usepackage{mathtools}
\usepackage{amsthm}
\usepackage{dblfloatfix} 

\usepackage[capitalize,noabbrev]{cleveref}

\usepackage[textsize=tiny]{todonotes}

\usepackage{colortbl}
\usepackage{makecell}

\icmltitlerunning{Adaptive Discrete Communication Bottlenecks with Dynamic Vector Quantization}

\begin{document}

\twocolumn[
\icmltitle{Adaptive Discrete Communication Bottlenecks \\ with Dynamic Vector Quantization}




\begin{icmlauthorlist}
\icmlauthor{Dianbo Liu}{yyy}
\icmlauthor{Alex Lamb}{yyy}
\icmlauthor{Xu Ji}{yyy}
\icmlauthor{Pascal Notsawo}{yyy}
\icmlauthor{Mike Mozer}{comp}
\icmlauthor{Yoshua Bengio}{yyy}
\icmlauthor{Kenji Kawaguchi }{sch}
\end{icmlauthorlist}

\icmlaffiliation{yyy}{Mila,Montreal,QC, Canada}
\icmlaffiliation{comp}{Google,CA, USA}
\icmlaffiliation{sch}{National University of Singapore, Singapore}

\icmlcorrespondingauthor{Dianbo Liu}{liudianbo@gmail.com}
\icmlcorrespondingauthor{Alex Lamb}{alex6200@gmail.com}

\icmlkeywords{Machine Learning, ICML}

\vskip 0.3in
]



\printAffiliationsAndNotice{\icmlEqualContribution} 

\begin{abstract}

Vector Quantization (VQ) is a method for discretizing latent representations and has become a major part of the deep learning toolkit. It has been theoretically and empirically shown that discretization of representations leads to improved generalization, including in reinforcement learning where discretization can be used to bottleneck multi-agent communication to promote agent specialization and robustness. 
The discretization tightness of most VQ-based methods is defined by the number of discrete codes in the representation vector and the codebook size, which are fixed as hyperparameters.
In this work, we propose learning to dynamically select discretization tightness conditioned on inputs, based on the hypothesis that data naturally contains variations in complexity that call for different levels of representational coarseness.  
We show that dynamically varying tightness in communication bottlenecks can improve model performance on visual reasoning  and reinforcement learning tasks. 

\end{abstract}

\section{Introduction}

Discretization of latent representations via vector quantization is a method for improving the robustness and generalization of learned models \cite{oord2017neural,liu2021discrete}. 
Replacing a continuous representation with a discrete representation, and limiting the capacity of the discrete representation, both improve generalization guarantees \cite{liu2021discrete}.
Discretization imposes a bottleneck \cite{tishby2000information} as the representation can take fewer values, and reducing capacity of the discrete representation tightens this bottleneck. 

Discretization can be used to bottleneck communication in modular inference models, for example in multi-agent reinforcement learning.
Modular inference combines outputs across modules into new module inputs while restricting the number of modules that communicate \cite{goyal2019recurrent}.
To make up for the loss in expressivity from being restricted to few communicative modules in each timestep, the model is forced to learn to perform inference in steps of specialized skill as opposed to applying all skills at once, which specializes the modules by definition since few are active in each step
\cite{darwen1996automatic,goyal2019recurrent,bengio2017consciousness,lamb2021transformers}. 
The communication bottleneck is tightened by switching from continuous to discrete representations \cite{liu2021discrete}. 
Beyond improvements in generalization error from decreasing capacity, it has been hypothesized that the inductive bias of bottlenecked communication between modules improves generalization by 1) reflecting the true causal structure of data generated by sparse interactions between few variables, 2) increasing robustness when modules are recombined in novel ways, compared to unspecialized modules with all-to-all communication, 3) improving sample efficiency since specialized modules require fewer training points to learn
\cite{goyal2019recurrent,bengio2017consciousness}.


However, discrete bottleneck methods typically lack adaptability. 
This work improves on the vector quantization method of \citet{liu2021discrete} by making tightness of the bottleneck dynamic. 
Instead of a single discretization function that maps all inputs to a discrete space with fixed size, we use a pool of discretization functions with varying levels of output capacity, and choose the function applied for a given input using key-query attention between representations of the input and discretization functions. 

The hypothesis is that this improves generalization because the optimal level of discretization suggested by the bounds, which is the tightest bottleneck such that training error can still be minimized \cite{liu2021discrete}, is unlikely to be the same for all regions of a data distribution.
The shortest description that captures adequate information in inputs for performing well on a task generally varies with the input, for example images contain different numbers of objects, and gameplay involves goals of varying complexity.
In terms of generalization error, using a single discretization capacity is potentially wasteful for simpler inputs, because the generalization gap can be reduced by selectively imposing a tighter representation bottleneck on the former without increasing training error.
To minimize the model selecting looser bottlenecks than necessary, we use an objective function that penalizes the choice of bottleneck proportional to its capacity.


In summary, the contributions of our
work are as follows:
\begin{itemize}
    \item We propose a dynamic vector quantization method (DVQ) that adaptively chooses the number of discrete codes and the codebook size that control tightness of the bottleneck.
    \item Our theoretical analysis shows that dynamic adjustment of the bottleneck improves generalization error under the sufficient condition of tighter average bottleneck and equal training loss.
    \item We empirically show improvement in performance by using DVQ to discretize inter-component communication within a deep learning model and inter-agent communication between agents, compared to using VQ with fixed bottleneck capacity. 
\end{itemize}

\section{Method}


\subsection{Communication discretization}

The process of converting data with continuous attributes into data with discrete attributes is called discretization \citep{chmielewski1996global}. In this study, we use discrete latent variables to quantize information communicated among different modules in a similar manner to codebooks \citep{liu2021discrete}, which is a general version of vector quantization (VQ-VAE, \citet{oord2017neural}) where the hidden representation is a list of discretized codes instead of a single discretized code. 
The discretization process for each vector $h  \in \Hcal\subset \RR^{m}$ is described as follows. First, vector $h$ is divided into $G$ segments 
$s_1,s_2,\dots,s_G$ with $h=\textsc{concatenate}(s_1,s_2,\dots,s_G),$ where each segment $s_i \in \RR^{m/G}$ with $\frac{m}{G} \in \NN^+$. Second, each continuous segment $s_i$ is discretized separately by being mapped to a discrete latent space vector $e \in \RR^{L \times (m/G)}$ where $L$ is the size of the discrete latent space (i.e., an $L$-way categorical variable):
$$
e_{o_i}=\textsc{discretize}(s_i), \quad \text{ where } o_i=\argmin_{j \in \{1,\dots,L\}} ||s_{i}-e_j||.
$$
These discretized codes, which we call the factors of continuous representation $h$, are concatenated to obtain the final discretized vector $z$ as

\begin{equation}
\begin{aligned}
z=\textsc{concatenate}(\textsc{discretize}(s_1),\\ \textsc{discretize}(s_2),...,\textsc{discretize}(s_G)).
\end{aligned}
\end{equation}


The multiple steps described above can be summarized by $z=q(h,L,G)$, 
where $q(\cdot)$ is the whole discretization process using the codebook, $L$ is the codebook size, and $G$\ is number of segments or factors per vector. 

The loss for model training is $\mathcal{L} = \mathcal{L}_\mathrm{task} + \mathcal{L}_{discretization}$ where
%
\begin{equation} \label{equation2}
\begin{aligned} 
 \mathcal{L}_{discretization} = \frac{1}{G} \bigg(\sum^{G}_{i}||\sg(s_{i})-e_{o_i}||^2_2
 \\ + \beta \sum^{G}_{i}||s_{i}-\sg(e_{o_i})||^2_2 \bigg).
\end{aligned}
\end{equation}
$\mathcal{L}_\mathrm{task}$ is the loss for the specific task, \textit{e.g.}, cross entropy loss for classification or mean square error loss for regression, $\sg$ refers to a stop-gradient operation that blocks gradients from flowing into its argument, and $\beta$ is a hyperparameter which controls the reluctance to change the code. The term $\sum^{G}_{i}||\sg(s_{i})-e_{o_i}||^2_2$ is the codebook loss, which only applies to the discrete latent vector and brings the selected $e_{o_i}$ close to the output segment $s_i$. The term $\sum^{G}_{i}||s_{i}-\sg(e_{o_i})||^2_2$ is the commitment loss, which only applies to the target segment $s_i$ and trains the module that outputs $s_i$ to make $s_i$ stay close to the chosen discrete latent vector $e_{o_i}$. Following the original codebooks and VQ-VAE papers, we found 0.25 to be a good value for $\beta$, and $e$ was initialized using $k$-means clustering on vectors $h$ with $k=L$.

\subsection{Dynamic bottlenecks}

Instead of a single codebook and discretization bottleneck, we use multiple bottlenecks where the tightness of each bottleneck is defined by the number of factors and codebook size. A pool of $N$ discretization functions 
$Q={\{q_t\}_{t \in [N]}}$
is made available to all representations being discretized, where the number of factors and codebook size for each discretization function are given by $G_t$ and $L_t$ respectively.
The discretization functions in the pool do not share parameters nor codebooks with each other. 
Each of the discretization functions is associated with a signature key vector $k_t \in \RR^{l}$
which is randomly initialized and learned in the training process. Key-value attention is conducted between $k_t$ and query $f(h) \in \RR^{l}$ where $h$ is the continuous representation being discretized and $f$ is a single layer neural network projector.
Gumbel-softmax \citep{jang2016categorical} is applied on the attention scores to make a one-hot selection of which $q_t$ to use, with categorical distribution $\pi^h$ over discretization functions for $h$ given by 
\begin{align}
    \pi^h(t) = \frac{\operatorname{exp}(k_t^\intercal f(h)) }{\sum_{j \in [N]} \operatorname{exp}( k_j^\intercal f(h))}.
\end{align}




\begin{figure*}[b]
    \centering
    \subfigure[Fixed bottleneck with a single discretization function \citep{liu2021discrete}.]{ \raisebox{7em}{\includegraphics[width=0.25\linewidth]{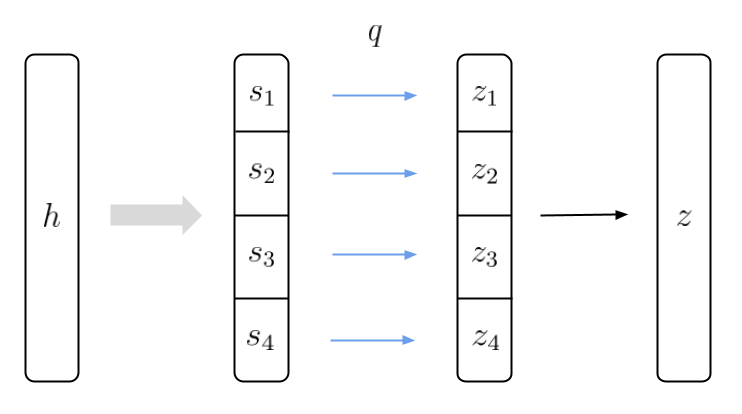}}
    }\hfill
    \subfigure[Dynamic bottleneck with flat attention over discretization functions.]{ \includegraphics[width=0.31\linewidth]{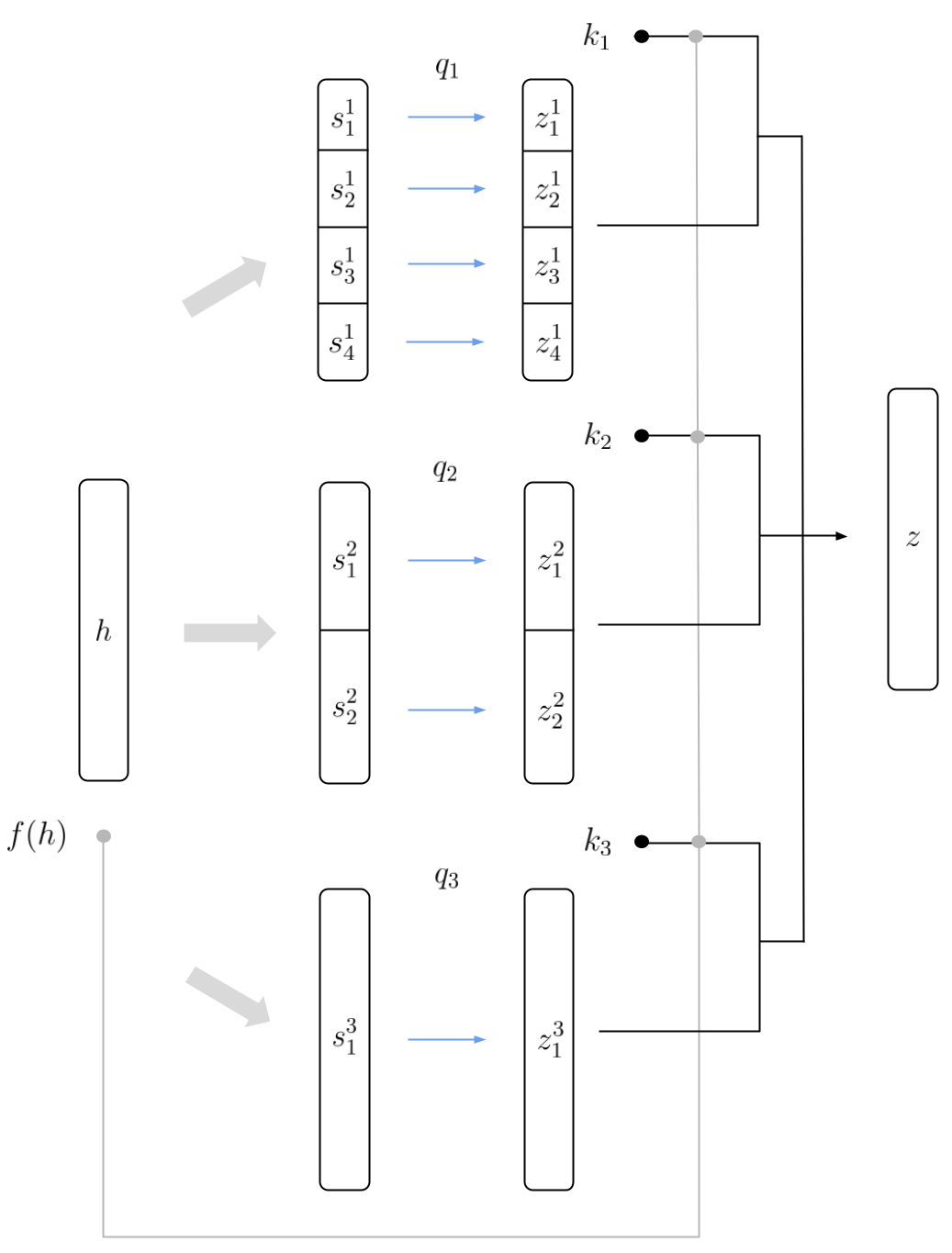}
    }\hfill
    \subfigure[Dynamic bottleneck with hierarchical attention over discretization functions.]{ \raisebox{3em}{\includegraphics[width=0.31\linewidth]{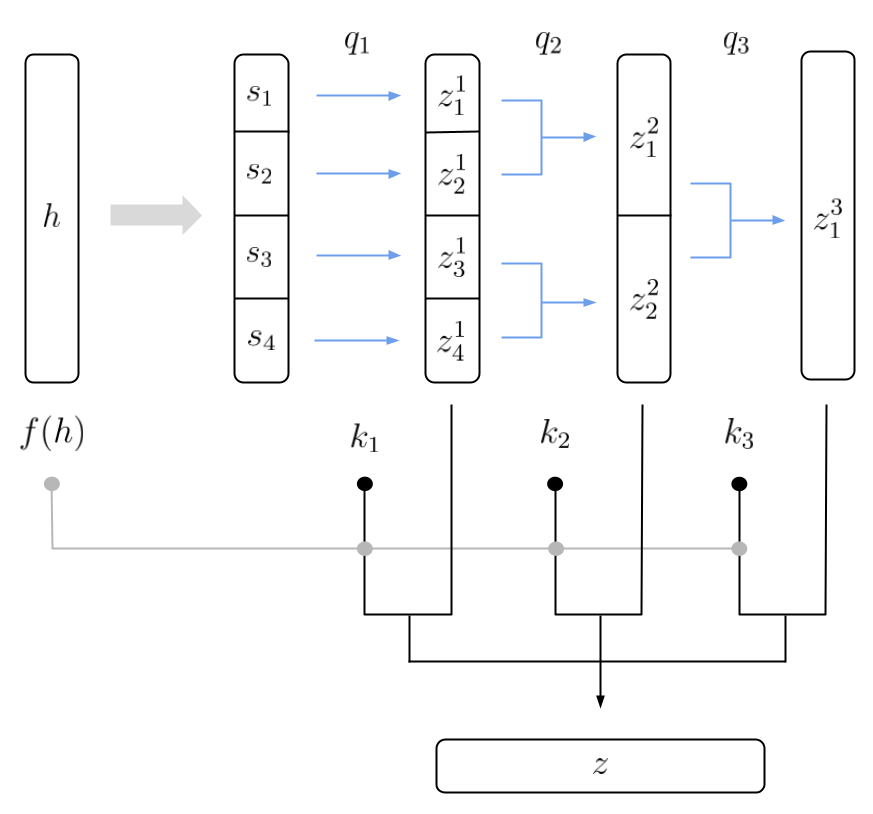}}
    }
    \caption{Dynamic discrete bottlenecks can be implemented as a flat function of continuous input $h$ (center), or as an iterative function that produces progressively coarser outputs (right). Bottleneck functions $q_t$ for $t \in [N]$, have different capacity and separate parameters. One bottleneck is selected for each input using key-value attention between representations of the input and discretization functions, $f(h)$ and $k_t$ respectively.}
    \label{fig:DynamicBottleneckingBoth}
\end{figure*}

\subsection{Bottleneck tightness}
In order to learn a communication bottleneck
with as few bits as necessary for minimizing training error,
we introduce pressure to choose $q_t$ with low $G_t$ and $L_t$. This pressure is implemented with capacity penalty $\mathcal{C}_{bottlenecking}$: 
\begin{align}
    &\mathcal{C}_{bottlenecking} = G_t \ln(L_t) \\
    \mathcal{L} = \mathcal{L}_{task} &+ \alpha \mathcal{L}_{discretization} +\beta \mathcal{C}_{bottlenecking}
\end{align}
where $\mathcal{L}$ is the overall loss for the representation being discretized, $\mathcal{L}_{task}$ is the task loss, $\mathcal{L}_{discretization}$ is the sum of commitment and codebook losses from the discretization process and $\mathcal{C}_{bottlenecking}$ is the bottlenecking cost. $\alpha$ and $\beta$ are hyperparameters that are chosen using a validation set. 

In our experiments, we found that including a continuous-valued function in the pool, corresponding to expressivity in the limit of codebook size $L_t \rightarrow \infty$, improved performance in some tasks. In this case $L_t$ was set to be a large number ($10^9$) in the penalty term $\mathcal{C}_{bottlenecking}$.


\subsection{Architectural choices}

Given continuous representation $h$, the bottleneck can be enforced in either a flat or hierarchical manner (\cref{fig:DynamicBottleneckingBoth}). 
In the flat case, the bottlenecked representation is the output of a single $q_t \in Q$ selected with key-value attention.
In the hierarchical case, all functions $\{q_t\}_{t \in [N]}$ are utilized by first ordering them in order of descending $G_t$, and setting the input segments of each function to be concatenated factors produced by the previous function, with $h$ as input into the first function.
Then output of a single $q_t \in Q$ is selected with key-value attention.




\section{Theoretical analysis}
 In this section, we  show that the adaptive discretization process has a potential advantage in improving the performance of the final model by better trading off the balance between the generalization and expressivity adaptively for each input. Moreover, our analysis predicts   that the additional regularization, $C_{bottlenecking}$, in our algorithm plays an important role for this tradeoff, and that the number of adaptive bottlenecks $N$ cannot be arbitrarily large. To achieve this goal, we follow  the abstract framework of \citet{liu2021discrete}. That is, let  $\phi$ be an arbitrary function, which  can refer to the composition of an evaluation criterion and the rest of the network following (adaptive) discretization bottlenecks.  Given any function $\phi: \RR^{m} \rightarrow \RR$ and any family of sets $S=\{S_1,\dots,S_K\}$ with $S_1,\dots,S_K \subseteq \Hcal$,  let us define the corresponding function $\phi_{k}^S$ by $\phi_{k}^S(h)=\one\{h \in S_k\}\phi(h)$ for all $k \in [K]$, where $[K]=\{1,2,\dots,K\}$.
Let $e^{(t)} \in  \RR^{L_t \times (m/G_t)}$ be fixed and we denote by $(Q_{k_{t}}^{(t)})_{k_{t} \in [L_{t}^{G_t}]}$ all the possible values after the discretization process for $t$-th adaptive bottleneck: i.e.,  $q(h,L_{t},G_{t})\in \cup_{k _{t}\in [L_{t}^{G_t}]} \{Q_{k_t}^{(t)}\}$ for all $h \in \Hcal$. We define $k_{\max}= \max_{t \in [N]} L_{t}^{G_t}$ and $Q_{k_t}^{(t)}=\emptyset$ for all $k_t > L_{t}^{G_t}$. Under this abstract setting with $t=N=1$, the previous paper \citep{liu2021discrete} proved their main theoretical result, showing that the models with non-adaptive discretization process have advantages over the continuous models without it (we present a slightly tighter version of the previous paper's results in Appendix \ref{app:proof}). Thus, in this section, we focus on the  comparison of  adaptive and non-adaptive discretization processes. 

To analyze the adaptive discretization process, we  introduce the additional notation: let $q_t(h)$ be the discretization process with a particular bottleneck $t \in [N]$,  and $q(h)$ be the whole discretization process as $q(h) = q_{\hht(h)} (h)$ where $\hht(h)$ is the result of the key-value attention. Define $I_{t}=\{i\in[n]: \hht(h_{i})=t\}$, which is the set of  the indices of training samples that end up using the $t$-th adaptive bottleneck.

The following theorem extends the main theorem from \citet{liu2021discrete} to the setting of adaptive bottlenecking and shows the advantage of the adaptive version over the non-adaptive version:

\begin{restatable}{theorem}{thma} \label{thm:1}
Let $N \in \NN_+$  and $S_{k}=\{Q_{k}^{(t)}\}_{t=1}^N$ for all $k \in [k_{\max}]$. Then, for any $\delta>0$, with probability at least $1-\delta$ over an iid draw of $n$ examples $(h_{i})_{i=1}^n$, the following holds for any $\phi: \RR^{m} \rightarrow \RR_{+}$ and $k \in [k_{\max}]$: 
\begin{align} \label{eq:thm:1:1}
 \EE_{h}[\phi_{k}^S(q(h))] &\le\frac{1}{n}\sum_{i=1}^n\phi_{k}^S(q(h_{i}))+\alpha (\mathcal{J}_{1} +\mathcal{J}_{2}),
\end{align}
where   $\alpha =\sup_{h \in \Hcal} \phi_k^S(h)$, 
\begin{align*}
&\mathcal{J}_{1} =\sum_{t=1}^N \frac{|I_{t}|}{n}  \sqrt{\frac{G_{t}\ln(L_{t})+\ln(N/\delta)}{2n}}, \text{ and }
\\ &  \mathcal{J}_{2}=\one\{N\ge 2\}\sqrt{\frac{2N \ln2 + 2 \ln(1/\delta)}{n}}.
\end{align*}
\end{restatable}
\begin{proof}
The proof is presented in Appendix \ref{app:proof}. 
\end{proof}
Theorem \ref{thm:1} recovers the previous result of \citep{liu2021discrete} when we set $N=1$ as desired. That is, by setting $N=1$, the inequality \eqref{eq:thm:1:1} in Theorem \ref{thm:1} becomes
\begin{align} 
 \EE_{h}[\phi_{k}^S(q(h))] &\le\frac{1}{n}\sum_{i=1}^n\phi_{k}^S(q_{1}(h_{i}))
 \\ \nonumber & \quad +  \alpha    \sqrt{\frac{G_{1}\ln(L_{1})+\ln(1/\delta)}{2n}},
\end{align}
which is the previous bound in Theorem 1 of \citet{liu2021discrete}. Thus, we successfully generalized the previous analysis framework to cover both the adaptive and  non-adaptive versions in an unified manner.

Theorem \ref{thm:1} shows that there are two ways that the adaptive version can improve the non-adaptive version; i.e., the potential improvement happens when $\sum_{t=1}^N \frac{|I_{t}|}{n}\sqrt{G_{t} \ln (L_t)} < \sqrt{G_{1} \ln(L_1)}$ or $\frac{1}{n}\sum_{i=1}^n\phi_{k}^S(q(h_{i})) <\frac{1}{n}\sum_{i=1}^n\phi_{k}^S(q_{1}(h_{i}))$.
The first criterion is met when the weighted average of the adaptive bottleneck sizes $G_{t} \ln (L_t)$ is smaller than the pre-fixed bottleneck size $G \ln(L)=G_{1} \ln(L_1)$. This is indeed encouraged by the additional regularization, $C_{bottlenecking}$, in our algorithm. The second criterion is satisfied when  the training loss with adaptive bottlenecks is less   than that with a fixed bottleneck. 

Theorem \ref{thm:1} provides the insight that the benefit of the adaptive bottlenecks
lies in the tradeoff between the expressivity (to minimize the training loss $\frac{1}{n}\sum_{i=1}^n\phi_{k}^S(q(h_{i}))$) and generalization (to minimize the term $\sum_{t=1}^N \frac{|I_{t}|}{n}\sqrt{G_{t} \ln (L_t)}$). 
For a fixed bottleneck, the training loss tends to decrease as $G_{1}$ increases because increasing $G_{1}$ improves the expressivity and trainability. However, increasing $G_{1}$ results in a worse bound on the generalization gap as the gap scales as $\sqrt{G_{1}/n}$ in Theorem 1. Thus, we have a tradeoff between the expressivity and generalization. For the adaptive  bottlenecks, different values of $G_{t}$ are used for different samples. As a result,
the adaptive  bottlenecks can have a better tradeoff
between expressivity and generalization by only using the necessary expressivity or bottleneck $G_t$ (and $L_t$) for each sample to reduce $\sum_{t=1}^N \frac{|I_{t}|}{n}\sqrt{G_{t} \ln (L_t)}$ while minimizing the training loss $\frac{1}{n}\sum_{i=1}^n\phi_{k}^S(q(h_{i}))$.

For example, consider a scenario with  a subset of training samples that require $G$ and $L$ to be extremely large to minimize the training loss for the subset. If we use a pre-fixed bottleneck, we need to make $G_{1}\ln(L_1)$ to be extremely large to minimize all training samples. This results in a bad generalization  term  $\sqrt{G_{1} \ln (L_{1})/n}$ in Theorem 1.
On the other hand, if we use the adaptive bottleneck, we can minimize the training loss  for all samples by  using a large $G_{t} \ln (L_t)$ only for the subset while using small values of $G_{t} \ln (L_t)$ for other samples.

In terms of these two criteria for the tradeoff,  the adaptive version seems to be always  better as we increase $N$. However, this better tradeoff comes with a cost. In Theorem \ref{thm:1}, the cost is captured by the additional term $\sqrt{\frac{2N \ln2 + 2 \ln(1/\delta)}{n}}$, which increases as $N$ increases. Thus, while the adaptive version is better in the  sense of the tradeoff of the two criteria, it comes with the additional cost of $\sqrt{N/n}$ term. This predicts that $N$ cannot be arbitrarily large.


\section{Related Work}


\paragraph{Vector Quantization.}
VQ is motivated by the fundamental result of Shannon’s rate-distortion theory \citep{VQ_signal_compression,element_info}: better performance can always be achieved by coding vectors instead of scalars, even if the sequence of source symbols are independent random variables. K-means \cite{macqueen1967some} is the prime method for VQ. The K-means algorithm clusters data by trying to separate the samples into $n$ groups of equal variance, this by minimizing the intra-class inertia. Despite the performance of this algorithm, K-means based VQ has an exponential complexity in encoding (computation and memory) and decoding (memory) \cite{Tan2018DeepVQAD}. Various improvements to K-means have been proposed to address complexity issues (e.g., product quantization) \cite{6678503}, but their performance is suboptimal compared to some deep neural network (DNN) based approaches \cite{Tan2018DeepVQAD}, where the idea is to map input data from the original high dimensional space to the DNN latent space with lower dimensionality, and apply K-means to the latent codes. Because K-means in the latent space is sub-optimal for VQ, \citet{Tan2018DeepVQAD} proposed DeepVQ, a fully-DNN architecture for vector quantization, in the context of data compression. DeepVQ is an autoencoder that overcomes the complexity issue by directly mapping to the binary index of the codeword.
Recent works \cite{rolfe2017discrete, maddison2017concrete} proposed novel reparameterization methods to handle the non-gradient issue for discrete random variables in VAE. VQ-VAE \cite{oord2017neural} avoids such problem by using the identity function, namely copying gradients from the decoder input to encoder output \cite{bengio2013estimating}.


\paragraph{Bottlenecking inter-module communication within a model.}

Many methods have been used in recent years to enable efficient communication between specialized components of machine learning models, from attention mechanisms for selectively communicating information between specialized components in machine learning models \cite{goyal2019recurrent,goyal2021coordination,goyal2021neural} and transformers \cite{vaswani2017attention,lamb2021transformers}; collective memory and shared parameters for multi-agent communication \cite{Garland1996MultiagentLT,Pesce2020}, node attributes in graph-based models \cite{10.5555/1795555,battaglia2018relational} for relational reasoning, dynamical systems simulation, multi-agent systems, and in many other areas. While most of inter-specialist communication mechanisms operates in a pairwise symmetric manner, \citet{goyal2021coordination} introduced a bandwidth limited communication channel to allow information from a limited number of modules to be broadcast globally to all modules, inspired by Global workspace theory \cite{Baars2019}. Recently, \citet{liu2021discrete} showed that replacing a continuous representation with a discrete representation, and limiting the discrete representation to a short list of codes from a small codebook, both improve generalization guarantees. Following VQ-VAE \cite{oord2017neural}, \citet{liu2021discrete} proposed discrete-valued neural communication (DVNC) to improve systematic generalization in a variety of architectures, including transformers \cite{vaswani2017attention,lamb2021transformers}, RIMs \cite{goyal2020recurrent} , and graph neural networks \cite{kipf2020contrastive}.


\paragraph{Bottlenecking inter-agent communication in multi-agent RL.}

A wide range of multi-agent applications have benefitted from inter-agent message passing including distributed smart grid control \cite{4840087}, consensus in networks \cite{5978201}, multi-robot control \cite{ren2008}, autonomous vehicle driving \cite{Petrillo2018}, elevators control \cite{10.1023/A:1007518724497} and for language learning in two-agent systems \cite{lazaridou2017multiagent}. An important challenge in MARL is how to facilitate communication among interacting agents, especially in tasks requiring synchronization \cite{SCARDOVI20092557, wen2012}. For example, in the multi-agent deep deterministic policy gradient \cite{lowe2020multiagent}, which extends the actor-critic algorithm \cite{degris2013offpolicy}; the input size of each critic increases linearly with the number of agents, which hinders its scalability \cite{jiang2018learning}. To overcome this, \citet{Pesce2020} provides the agents with a shared communication device that can be used to learn from their collective private observations and share relevant messages with others in the centralised learning and decentralised execution paradigm \cite{foerster2016learning, KRAEMER201682, 10.5555/1622673.1622680}. However, their approach is limited to small-scale systems.  \citet{iqbal2019actorattentioncritic} proposed Multi-Actor-Attention-Critic to learn decentralised policies with centralised critics, which selects relevant information for each agent at every time-step through an attention mechanism. Many other communication mechanisms have been proposed, such as CommNet \cite{sukhbaatar2016learning}, IC3NEt \cite{singh2018learning},  BiCNet \cite{peng2017multiagent}, attention \cite{jiang2018learning} and soft-attention \cite{das2020tarmac} based, master-slave architecture \cite{kong2017revisiting}, Feudal Multiagent Hierarchies \cite{ahilan2019feudal}, Bayesian Action Decoder \cite{foerster2019bayesian}.


\section{Experiments}
In our experiments we test the hypothesis that imposing a dynamic bottleneck on communication facilitates modularization and improved generalization compared to a fixed bottleneck.  
We test our method in two settings, 1) inter-component communication among different components in the same model and  2) inter-agent communication in cooperative multi-agent reinforcment learning (MARL).

\subsection{Experimental setup}

\paragraph{Discretize inter-module communication within a model.} We explore the effects of using DVQ to bottleneck inter-module communication for visual reasoning tasks. The tasks are to predict object movement in a grid world (referred to as \textbf{``Shapes''}) and 7 different Atari game environments. In all of the environments, changes in each image frame depend on the previous image frame and the actions applied to different objects. Objects are captured by a convolutional neural network and are passed to a Graph neural network (GNN). Positions of objects are captured by nodes in the GNN and the relative positions among different objects are communicated through the network a discrete manner \citep{kipf2020contrastive}. In this work we apply either DVQ or VQ to discretize the vector sum of edges each node is connected to.

\paragraph{Bottleneck communication among reinforcement learning agents.}
To investigate the potential of using DVQ to bottleneck communication among different reinforcement learning agents, we conducted experiments in two MARL environments with cooperative tasks. Multi-Agent Particle Environment is a 2D cooperative MARL environment originally proposed in the MADDPG paper \cite{lowe2020multiagent}. There are 3 agents and 3 landmarks in a 2D space. In this study, we use the cooperative navigation task (called \textbf{``SimpleSpread''}) where agents are required to cover all landmarks while minimizing inter-agent collisions. The action space is discrete and agents can move either up, down, left or right in addition to a no-move ``action''. Each agent only has a partial view of the environment.
In this work, we allow agents to encode their partial view and send it as messages to other agents.  We apply discretization bottlenecks on the messages each agent receives. The second MARL environment we use is GhostRun, which is an adaptation the environment available from \citet{shuo2019maenvs}. The environment consists of multiple agents, each with a partial view of the ground below them. There are multiple ghosts (red dots) moving randomly in the environment. All the agents work as a team to escape from the ghosts. Once an agent has ghosts in its view, the team receives negative rewards which are multiplied by the total number of ghosts in the joint view of all agents. Similar to the Particle environment, we allow agents to communicate with each other by sending their encoded partial view as messages, and we discretize the messages received by each agent.

\begin{figure}
    \centering

    \includegraphics[width=0.40\linewidth]{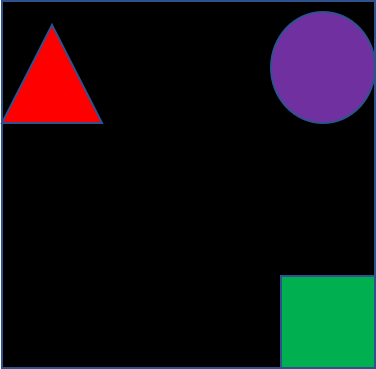}
    \includegraphics[width=0.355\linewidth]{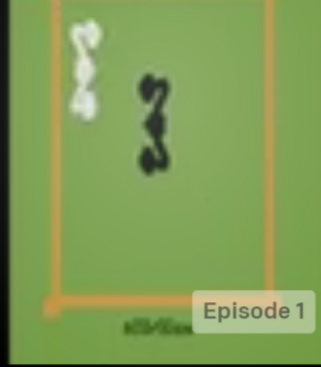}
    \includegraphics[width=0.40\linewidth]{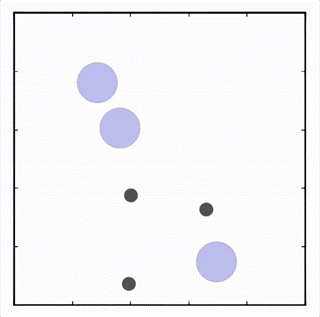}
    \includegraphics[width=0.40\linewidth]{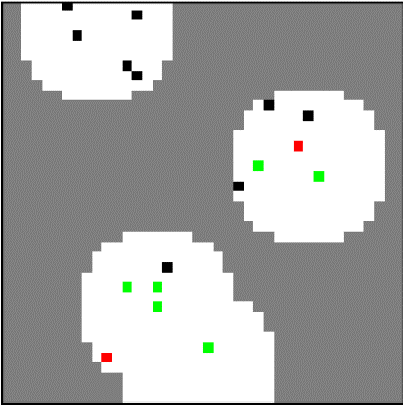}
    \caption{Two visual reasoning tasks, grid word (top left) and 7 Atari games (top right), are used to analyze effects of bottlenecking inter-module communication within a model using DVQ. Two MARL tasks, Particles (bottom left) and GhostRun (bottom right), are used to investigate the effect of bottlecking inter-agent communication using DVQ.}
    \label{fig:MARLParticle}
\end{figure}

\subsection{Inter-module communication}
The bottlenecks were applied on edges of graphs in the GNN linking nodes that represent objects in the image  \cite{liu2021discrete}. The visual reasoning tasks require the model to forecast future scenes in the environments based on current state and actions if available. 
Out of the 8 visual reasoning tasks, DVQ outperforms fixed-bottleneck VQ in 7 tasks (figure \ref{table:result_iid}), with the adaptive hierarchical architecture best on average.


\begin{table*}
\label{table:result_iid}
\caption{Performance of different methods in bottlenecking inter-module communication in a visual reasoning model}
\centering
		\arrayrulecolor{black}
		\begin{tabular}{!{\color{black}\vrule}c!{\color{black}\vrule}c!{\color{black}\vrule}c!{\color{black}\vrule}c!{\color{black}\vrule}c!{\color{black}\vrule}}
		\arrayrulecolor{black}\hline
				Task/Model           & \textbf{Original} & \textbf{Quantization} & \textbf{Adaptive Quantization} & \textbf{Adaptive Hierarchical} \\ \hline
				Alien         & 0.130 $\pm$ 0.023 & 0.152 $\pm$ 0.026 & 0.170 $\pm$ 0.075 & \textbf{0.177 $\pm$ 0.057} \\ \hline
				BankHeist     & 0.397 $\pm$ 0.043 & 0.371 $\pm$ 0.057 & 0.406 $\pm$ 0.037 & \textbf{0.414 $\pm$ 0.084} \\ \hline
				Berzerk       & 0.436 $\pm$ 0.250 & 0.584 $\pm$ 0.011 & \textbf{0.630 $\pm$ 0.016} & 0.580 $\pm$ 0.021 \\ \hline
				Boxing        & 0.873 $\pm$ 0.021 & 0.908 $\pm$ 0.068 & 0.929 $\pm$ 0.031 & \textbf{0.957 $\pm$ 0.041} \\ \hline
				MsPacman      & \textbf{0.152 $\pm$ 0.037} & 0.135 $\pm$ 0.030 & 0.054 $\pm$ 0.002 & 0.057 $\pm$ 0.005 \\ \hline
				Pong          & 0.169 $\pm$ 0.047 & 0.201 $\pm$ 0.035 & 0.205 $\pm$ 0.068 & \textbf{0.225 $\pm$ 0.031} \\ \hline
				shapes        & 0.674 $\pm$ 0.055 & 0.672 $\pm$ 0.053 & 0.664 $\pm$ 0.034 & \textbf{0.692 $\pm$ 0.065} \\ \hline
    		    SpaceInvaders & 0.138 $\pm$ 0.037 & 0.199 $\pm$ 0.085 & \textbf{0.258 $\pm$ 0.103} & 0.232 $\pm$ 0.076 \\ \hline
		\end{tabular}  
\end{table*}

\begin{figure*}
    \centering
    \includegraphics[width=0.45\linewidth]{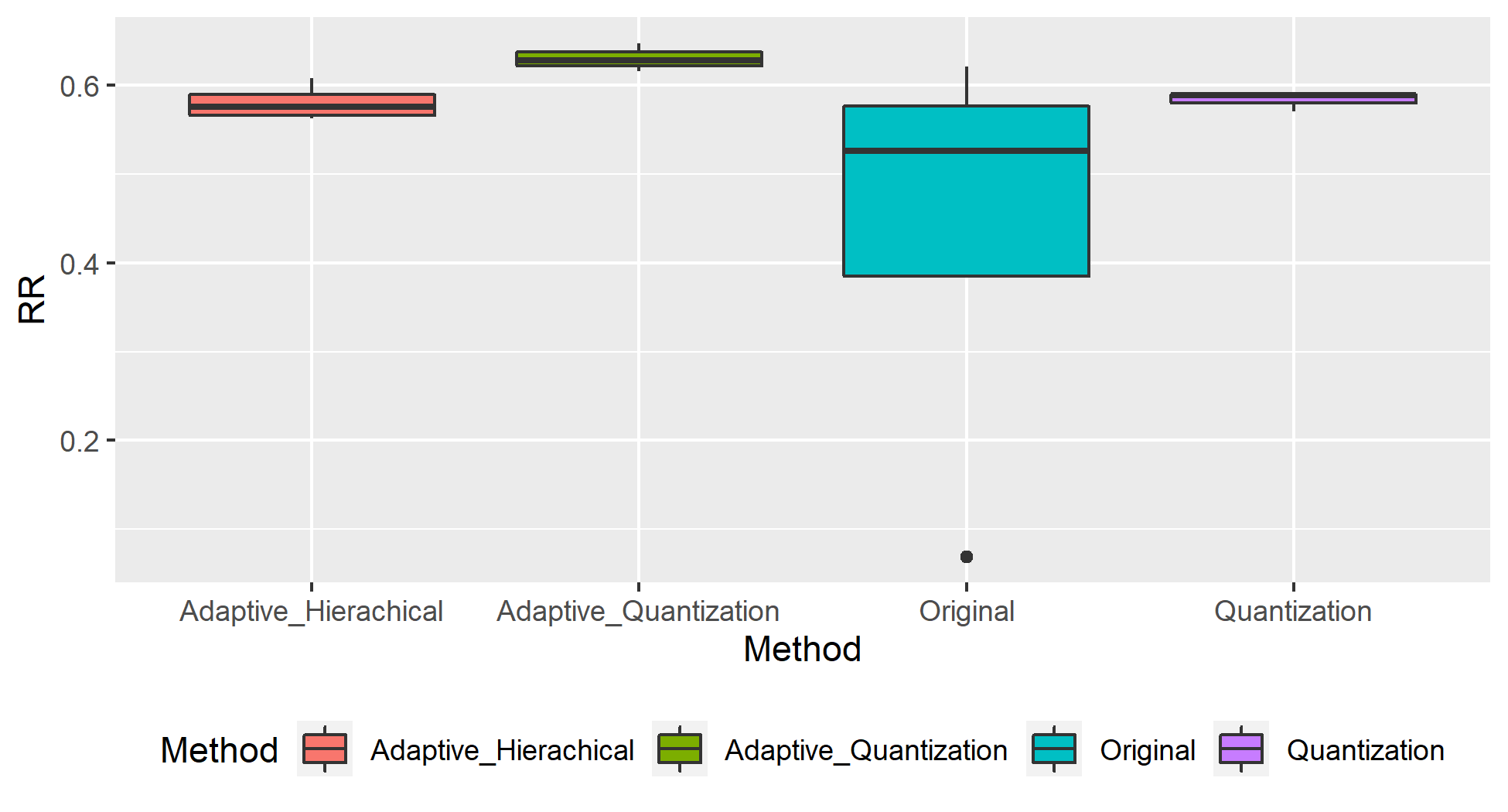}
    \includegraphics[width=0.45\linewidth]{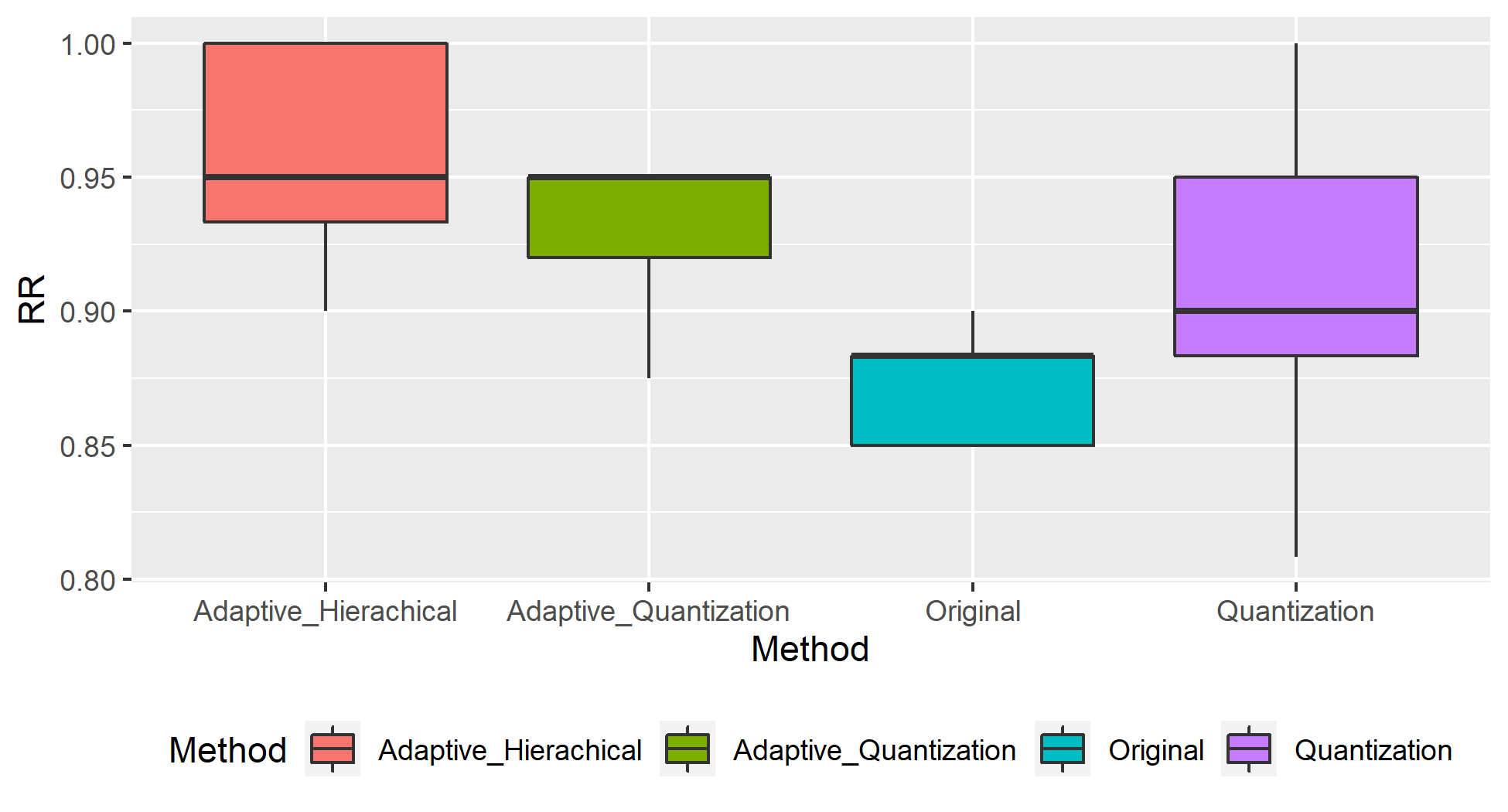}
    \caption{Bottlenecking communication among components in a model in visual reasoning task. This figure shows examples of two games. Left panel: performance of different methods in Berzerk game. Right panel: performance of different methods in Boxing game }
    \label{fig:VisualReasoning}
\end{figure*}

\subsection{Inter-agent communication in cooperative MARL}
Next we investigate the potential benefits of using DVQ to bottleneck information exchange among multiple agents in two cooperative MARL tasks. Agents communicate with each other in a broadcasting manner where each agent $j$ sends identical messages $m_{j,t}$, the encoded partial observation of agent $j$, to all other team members simultaneously. Each agent in the cooperative game receive $M_t=\{m_{j,t}|0\leq j < N_{agents}\}$ from all agents in the team at the same time and reads $M_t$ using an attention mechanism \mbox{$m'_{i,t} = \mathrm{softmax} \bigg( \frac{q^{m}_{i, t} \vkappa_{t}
   }M*{t}{\sqrt{d_m}}\bigg)$}
where $m'_{i,t}$ is the final message agent $i$ received from others at time step $t$ and $q^{m}_{i, t}=f_q (m_{j,t})$ and 
$\vkappa_{t}=f_m (M*{t})$. Both $f_m$ and $f_q$ are MLP encoders. At each time step, for each agent, $m'_{i,t}$ is discretized. In DVQ, $L \in [16,64,256]$ and $G \in [1,2,4]$.
In the VQ baseline, $L=16$ and $G=1$. In the Particle environment, both dynamic hierarchical quantization and dynamic quantization outperform VQ with fixed coarseness. In the GhostRun environment, dynamic hierarchical quantization outperforms the baseline and dynamic quantization ties with the baseline (Figure~\ref{fig:MARLGhostRun}).

\begin{figure}
    \centering

    \includegraphics[width=0.48\linewidth]{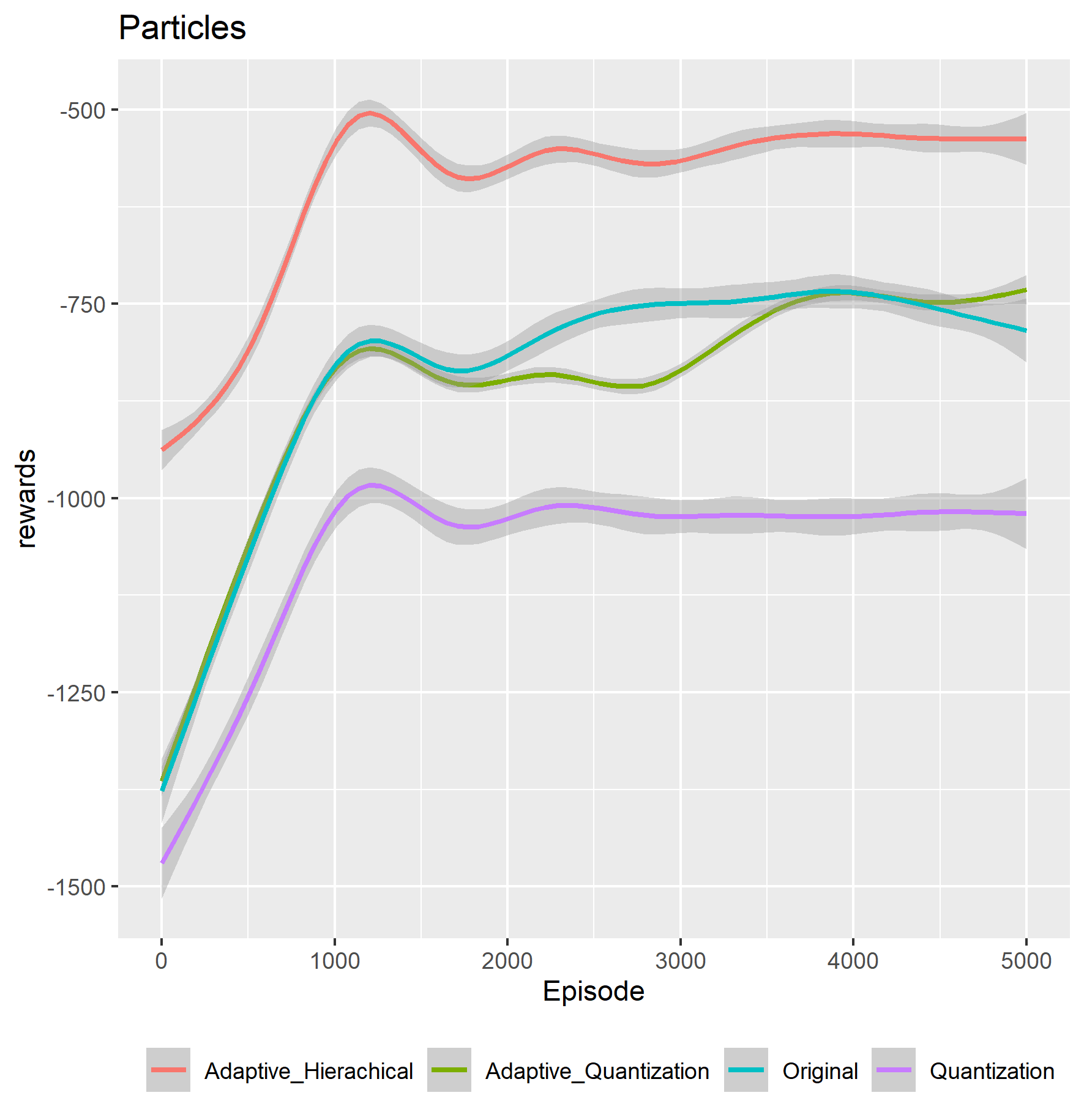}

    \caption{Bottlenecking communication among cooperative agents in Particles environment.  }
    \label{fig:MARLParticle}
\end{figure}

\begin{figure}
    \centering

    \includegraphics[width=0.98\linewidth]{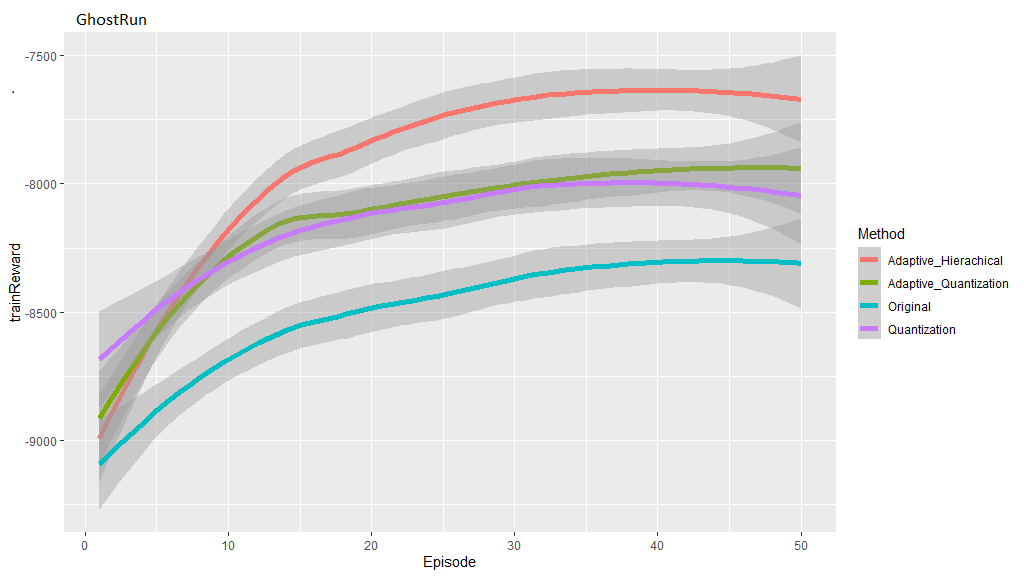}

    \caption{The effect of applying  DVQ to bottleneck communication among cooperative agents in GhostRun environment.  }
    \label{fig:MARLGhostRun}
\end{figure}

\subsection{Task difficulty and bottleneck tightness}
In the sections above, our experiments showed that DVQ outperforms VQ with fixed capacity in various tasks. Next, we seek to understand behaviors of DVQ. In the visual reasoning tasks, difficulty varies significantly among different games. For example, object movement in some Atari games are more unpredictable than others. In addition, even within the same game, difficulty many vary among episodes. This motivates us to ask the question of whether the tightness of bottleneck introduced by DVQ is influenced by how hard a task is.  To answer this, we first quantify difficulty of the task based on the test performance of the visual reasoning model. Reciprocal Ranks (RR) is used as the performance metric where  higher RR means better performance. Next, we break down the capacity penalty into the factor loss, which corresponds to penalty as a result of the number of factors, and codebook size loss, which is the term that penalizes high capacity caused by large codebook sizes. Our analysis shows that both factor loss and codebook size loss have positive correlation with RR, across the 8 tasks (Figure \ref{fig:RRCorLoss}), with stronger correlation in the case of number of factors.
Depending on the direction of causality in this observation, it suggests that tighter bottlenecks are picked for hard tasks either because of difficulty of optimization (when the optimal level of bottleneck tightness is not found, the model tends to undershoot capacity) or because harder tasks call for stronger regularization from tighter bottenecks. 

\begin{figure}
    \centering

    \includegraphics[width=0.8\linewidth]{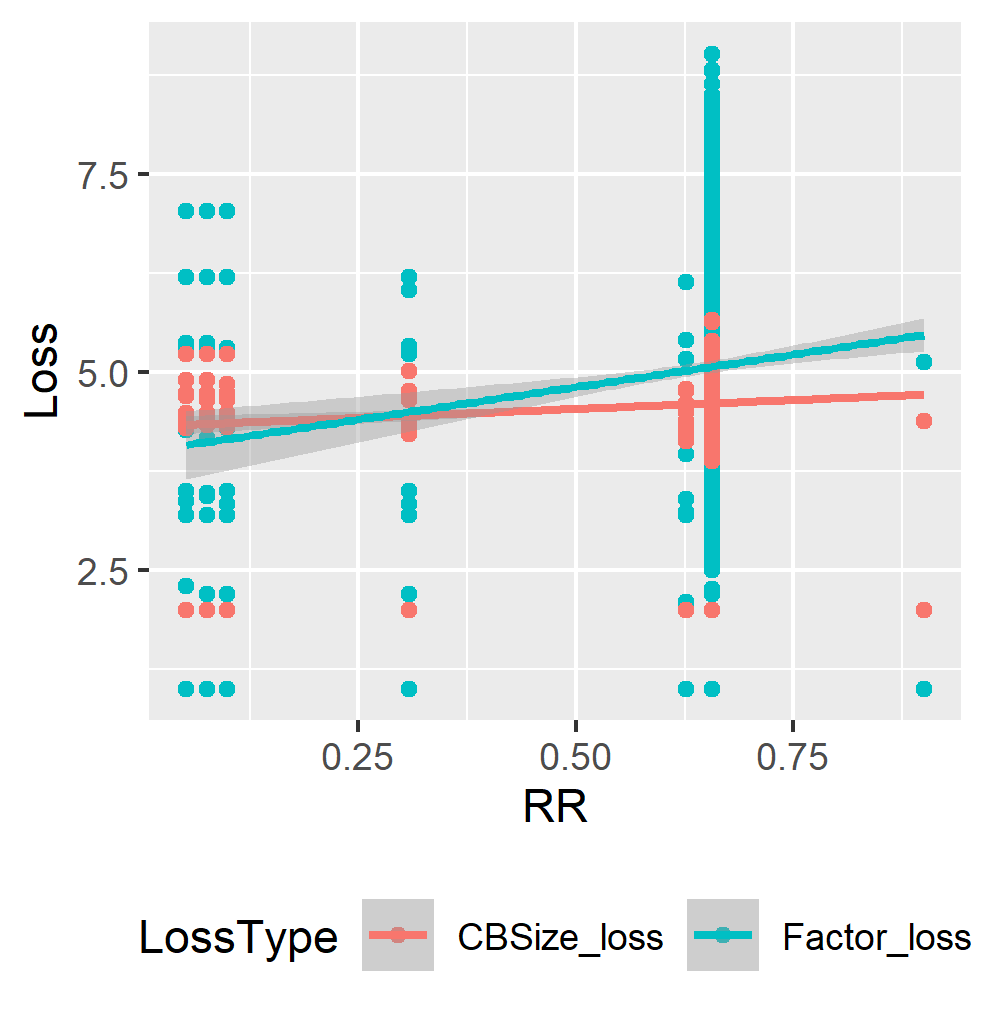}
    
    \caption{Tighter bottlenecks are used in difficult tasks. A positive correlation of both factor size (blue) and codebook size (red) with model performance measured by reciprocal rank (RR). The higher the RR, the better the model performs in the task.}
    \label{fig:RRCorLoss}
\end{figure}

\subsection{Discretization difficulty and bottleneck tightness}
In addition to difficulty of the task itself, another dimension to consider is the difficulty of discretization, which is largely determined by distribution of the learned representation and dynamics of the model. We estimate difficulty of discretization using the discretization loss obtained during evaluation. Recall that the discretization loss is averaged over factors. We observed that DVQ tends to increase its capacity by increasing both the number of factors and codebook size when the discretization loss is high (figure \ref{fig:DisCor}), which is reasonable as high expressivity in inputs puts positive pressure on decreasing the tightness of the bottleneck. 

\begin{figure}
    \centering
    \includegraphics[width=0.45\linewidth]{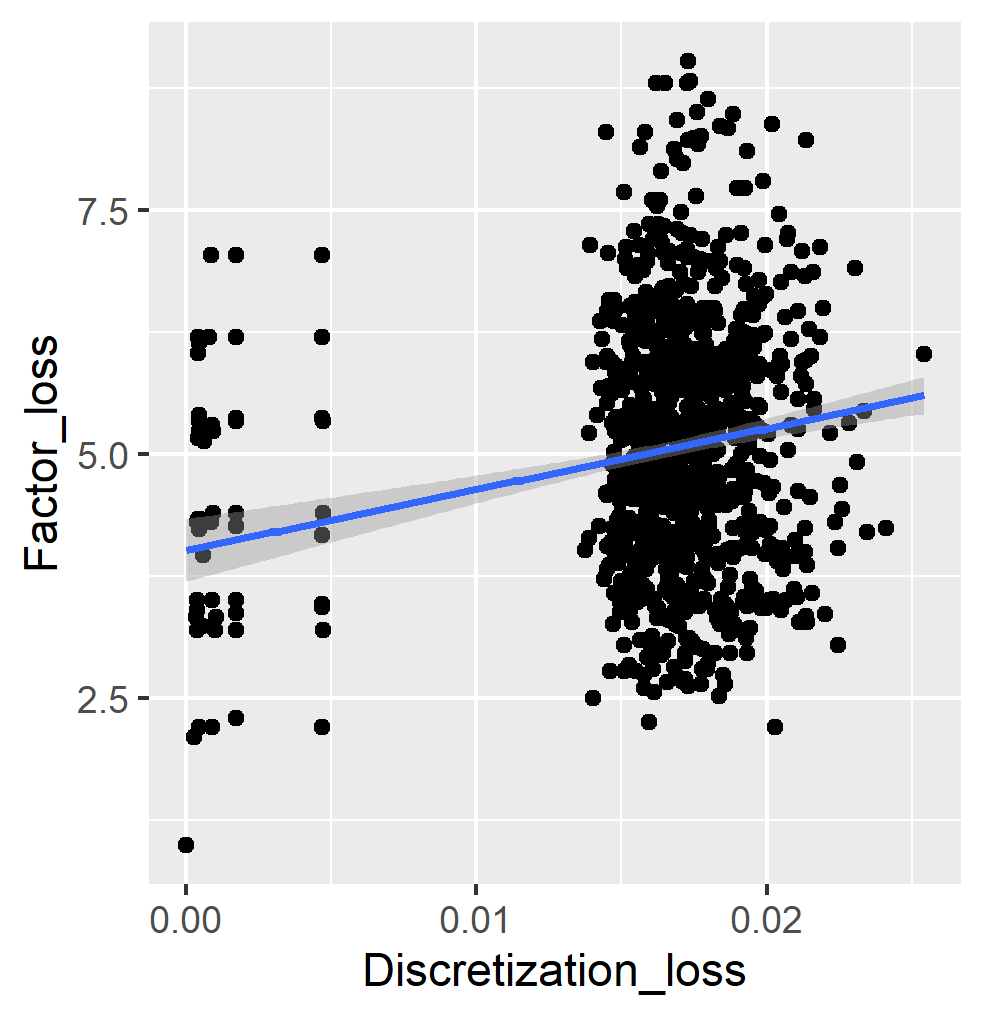}
    \includegraphics[width=0.45\linewidth]{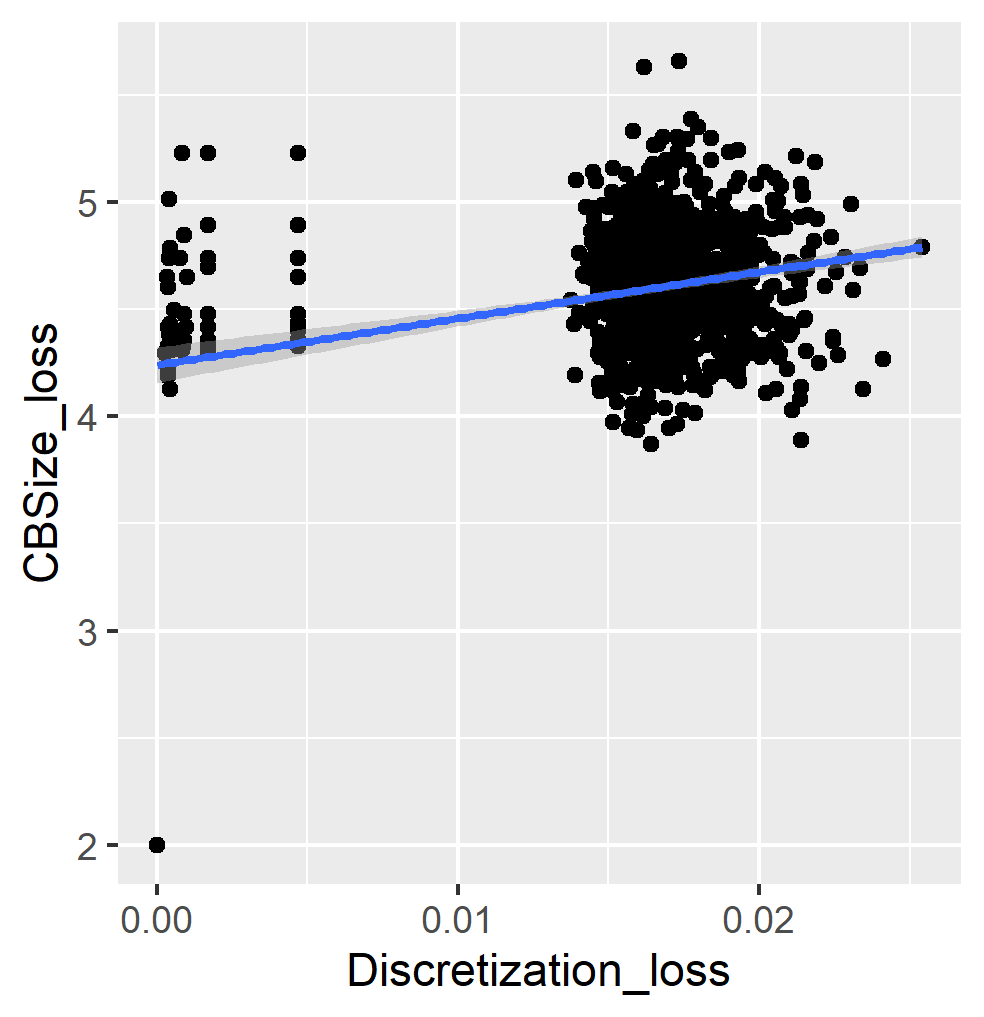}
    \caption{More factors and larger codebook sizes are used by DVQ when discretization is difficult.}
    \label{fig:DisCor}
\end{figure}


\section{Conclusion}

Effective communication between different specialized components of a model or between different agents in a MARL system requires compatible representations and synchronized messages. 
Although communication via continuous, high-dimensional signals is a natural choice given the ease of optimization, recent work has shown that discretization can provide more robust communication. 
Creating a communication bottleneck with pressure to utilize only as few bits as possible needed to coordinate inference improves generalization and imposes disentangling of modules and specialization, which is beneficial for training (learning $N$ skills independently by $N$ experts is easier than learning an organization of knowledge where skills are entangled and all skills interact with each other).
Among these works, those based on VQ depend on the number of discrete codes in the representation vector and the size of the codebook, which are fixed, limiting their adaptability and their ability to fit bottleneck behaviour tightly to data.  

We have shown theoretically and empirically the benefits of using a set of discretization functions with varying levels of output capacity and choosing the function applied for a given input using key-query attention, instead of using a single discretization function that maps all inputs to a discrete space of fixed size.
Adaptively bottlenecking capacity makes tightness dependent on inputs, which allows the generalization gap to be decreased. The experiments show that performance is improved compared to using a fixed capacity bottleneck across a range of tasks in visual reasoning and multi-agent reinforcement learning.



\clearpage
\nocite{langley00}

\bibliography{example_paper}
\bibliographystyle{icml2022}

\newpage
\appendix
\onecolumn


\section{Theoretical Analysis} \label{app:proof}

\subsection{Improvement of previous bounds}

Under the same setting as our analysis, the previous paper \citep{liu2021discrete} proved their main theoretical result for non-adaptive discretization process. In the below, we slightly improve  the previous result: i.e., the  the following two propositions are slightly tighter versions of the main theoretical results from the previous paper \citep{liu2021discrete} on the expected loss:

\begin{restatable}{proposition}{propa} \label{prop:1}
\emph{(with non-adaptive discretization)} Let $S_{k}=\{Q_{k}\}$ for all $k \in [L^{G}]$. Then, for any $\delta>0$, with probability at least $1-\delta$ over an iid draw of $n$ examples $(h_{i})_{i=1}^n$, the following holds for any $\phi: \RR^{m} \rightarrow \RR_{+}$ and all $k \in [L^{G}]$:  
\begin{align} \label{eq:prop:1:1}
 \EE_{h}[\phi_{k}^S(q(h,L,G))]  = \frac{1}{n}\sum_{i=1}^n\phi_{k}^S(q(h_{i}, L,G)) + \Ocal\left( \sqrt{\frac{G\ln(L)+\ln(1/\delta)}{2n}}\right), 
\end{align}
where   $\alpha =\sup_{h \in \Hcal} \phi_k^S(h)$, and no constant is hidden in $\Ocal$.
\end{restatable}
\begin{proof}[Proof of Proposition \ref{prop:1}]
The desired statement follows from the exact same proof of  Theorem 1 of \citep{liu2021discrete} except that we replace the Hoeffding's inequality
on the absolute value of the difference $\left|p_k -\hat p_k\right|$ with that on the difference $(p_k -\hat p_k)$. 
\end{proof}

\begin{restatable}{proposition}{propb} \label{prop:2} 
\emph{ (without discretization)}  Assume that  $\|h\|_2 \le R_{\Hcal}$ for all $h \in \Hcal\subset \RR^{m}$. Fix  $\Ccal\in \argmin_{\bCcal}\{|\bCcal|:\bCcal \subseteq \RR^{m},\Hcal \subseteq \cup_{c \in\bCcal} \Bcal_{}[c]\}$ where $\Bcal[c]=\{x\in \RR^{m} : \|x - c \|_2\le R_{\Hcal}/(2\sqrt{n})\}$. Let $S_{k}=\Bcal[c_{k}]$ for all $k \in [|\Ccal|]$ where $c_k \in \Ccal$ and $\cup_{k} \{c_k\}=\Ccal$.   Then, for any $\delta>0$, with probability at least $1-\delta$ over an iid draw of $n$ examples $(h_{i})_{i=1}^n$, the following holds  for any $\phi: \RR^{m} \rightarrow \RR_{+}$  and  all $k \in [|\Ccal|]$:
if $|\phi_k^S(h)|\le \alpha$ for all $h \in \Hcal$ and 
$|\phi_{k}^{S}(h{})-\phi_{k}^{S}(h{}')|\le \varsigma_{k}\|h -h'\|_2$ for all $h,h' \in S_{k}$, then
\begin{align} \label{eq:prop:2:1}
 \EE_{h}[\phi_{k}^S(h)]  
 = \frac{1}{n}\sum_{i=1}^n\phi_{k}^S(h_{i}) +\Ocal\left(\alpha \sqrt{\frac{m\ln(4  \sqrt{n m})+\ln(1/\delta)}{2n}}+\frac{\bar \varsigma_{k}R_{\Hcal}}{\sqrt{n}}\right), \end{align}
where  $\bar \varsigma_{k}= \varsigma_{k} \left(\frac{1}{n} \sum_{i=1}^n\one\{h_{i}\in \Bcal_{}[c_{k}]\}\right)$, $\alpha =\sup_{h \in \Hcal} \phi_k^S(h)$, and no constant is hidden in $\Ocal$.
\end{restatable}
\begin{proof}[Proof of Proposition \ref{prop:2}]
The desired statement follows from the exact same proof of  Theorem 2 of \citep{liu2021discrete} except that we replace the Hoeffding's inequality
on the absolute value of the difference $\left|p_k -\hat p_k\right|$ with that on the difference $(p_k -\hat p_k)$. 
\end{proof}

\subsection{Proof of Theorem \ref{thm:1}}

\thma*
\begin{proof}[Proof of Theorem \ref{thm:1}]   We can decompose the loss as\begin{align*}
\frac{1}{n}\sum_{i=1}^n\phi_{k}^S(q(h_{i})) &=\frac{1}{n}\sum_{i=1}^n\phi_{k}^S(q_{\hht(h_{i})}(h_{i}))
\\ & =\frac{1}{n} \sum_{t=1}^N \sum_{i\in I_{t}}\phi_{k}^S(q_{\hht(h_{i})}(h_{i})) 
\\ & =\frac{1}{n} \sum_{t=1}^N |I_{t}| \frac{1}{|I_{t}|}\sum_{i\in I_{t}}\phi_{k}^S(q_{t}(h_{i})) 
\end{align*}
Similarly,
\begin{align*}
\EE_{h}[\phi_{k}^S(q(h))] =\sum_{t=1}^N\Pr(\hht(h)= t)\EE_{h}[\phi_{k}^{S}(q_t (h_{}))| \hht(h)= t]. 
\end{align*}
Combining these,
\begin{align*}
&\EE_{h}[\phi_{k}^S(q(h))] - \frac{1}{n}\sum_{i=1}^n\phi_{k}^S(q(h_{i}))
\\ &=\sum_{t=1}^N\Pr(\hht(h)= t)\EE_{h}[\phi_{k}^{S}(q_t (h_{}))| \hht(h)= t]-\frac{1}{n} \sum_{t=1}^N |I_{t}| \frac{1}{|I_{t}|}\sum_{i\in I_{t}}\phi_{k}^S(q_{t}(h_{i}))
\\ & =\sum_{t=1}^N\left(\Pr(\hht(h)= t)\EE_{h}[\phi_{k}^{S}(q_t (h_{}))| \hht(h)= t]-  \frac{|I_{t}|}{n} \frac{1}{|I_{t}|}\sum_{i\in I_{t}}\phi_{k}^S(q_{t}(h_{i})) \right) 
\\ & =\sum_{t=1}^N\left(\Pr(\hht(h)= t)\EE_{h}[\phi_{k}^{S}(q_t (h_{}))| \hht(h)= t]-\frac{|I_{t}|}{n}\EE_{h}[\phi_{k}^{S}(q_t (h_{}))| \hht(h)= t]\right)
\\ &\ \qquad +\sum_{t=1}^N \left( \frac{|I_{t}|}{n}\EE_{h}[\phi_{k}^{S}(q_t (h_{}))| \hht(h)= t]-  \frac{|I_{t}|}{n} \frac{1}{|I_{t}|}\sum_{i\in I_{t}}\phi_{k}^S(q_{t}(h_{i})) \right) 
\\ & =\sum_{t=1}^N \EE_{h}[\phi_{k}^{S}(q_t (h_{}))| \hht(h)= t]\left(\Pr(\hht(h)= t)-\frac{|I_{t}|}{n}\right)
 +\sum_{t=1}^N \frac{|I_{t}|}{n} \left( \EE_{h}[\phi_{k}^{S}(q_t (h_{}))| \hht(h)= t]-   \frac{1}{|I_{t}|}\sum_{i\in I_{t}}\phi_{k}^S(q_{t}(h_{i})) \right) 
\end{align*}
Since $\phi_{k}^{S}(q_t (h_{}))=0$ when $ \hht(h)= t$ and $k >L_t^{G_t}$, by defining $\Tcal_{k}=\{t \in [N]: k \le L_t^{G_t}\}$, 
\begin{align} \label{eq:15}
&\EE_{h}[\phi_{k}^S(q(h))] - \frac{1}{n}\sum_{i=1}^n\phi_{k}^S(q(h_{i})) 
\\ \nonumber &\le\alpha\sum_{t=1}^N \left|\Pr(\hht(h)= t)-\frac{|I_{t}|}{n}\right|+\sum_{t\in \Tcal_{k}} \frac{|I_{t}|}{n} \left( \EE_{h}[\phi_{k}^{S}(q_t (h_{}))| \hht(h)= t]-   \frac{1}{|I_{t}|}\sum_{i\in I_{t}}\phi_{k}^S(q_{t}(h_{i})) \right). 
\end{align}
For the first term, the vector $(|I_{1}|,\dots,|I_{N}|)$ follows  a multinomial distribution with parameters $n$ and $p=(p_1, ..., p_{N})$, where $p_t=\Pr(\hht(h)= t) $ for $t=1,\dots,N$. Thus, by using the Bretagnolle-Huber-Carol inequality \citep[A6.6 Proposition]{van1996}, we have that with probability at least $1-\delta$,
\begin{align} \label{eq:16}
\alpha\sum_{t=1}^N \left|\Pr(\hht(h)= t)-\frac{|I_{t}|}{n}\right| \le\alpha\one\{N\ge 2\} \sqrt{\frac{2N \ln2 + 2 \ln(1/\delta)}{n}},
\end{align}
where we have the factor $\one\{N\ge 2\}$ because if $N=1$,
$$
\sum_{t=1}^N \left|\Pr(\hht(h)= t)-\frac{|I_{t}|}{n}\right| =1-1= 0.
$$
We now proceed to analyze the second term. 
Fix $t\in[N]$. Define $\Ical_{k}=\{i\in I_{t}: q_t (h_{i})= Q_{k}^{(t)}\}$ and $h_t$ to be the random variable that follow the conditional distribution: i.e., $\EE_{h_{t}}[\phi_{k}^{S}(q_t (h_{t}))]=\EE_{h}[\phi_{k}^{S}(q_t (h_{}))| \hht(h)= t]$.  
Then, we can further decompose the expectation with the conditional expectations as 
$$
\EE_{h_t}[\phi_{k}^{S}(q_t (h_t))]=\EE_{h_t}[\phi_{k}^{S}(q_t (h_t))|q_t (h_t)= Q_{k}^{(t)}]\Pr(q_t (h_t)= Q_{k}^{(t)})=\phi(Q_{k}^{(t)})\Pr(q_t (h_t)= Q_{k}^{(t)}).
$$
Using this, we can decompose the difference into two terms as
\begin{align} \label{eq:11}
&\EE_{h}[\phi_{k}^{S}(q_t (h_{}))| \hht(h)= t]-   \frac{1}{|I_{t}|}\sum_{i\in I_{t}}\phi_{k}^S(q_{t}(h_{i})) 
\\ \nonumber & =\phi(Q_{k}^{(t)})\left(\Pr(q_t (h_t)= Q_{k}^{(t)}) - \frac{|\Ical_{k}|}{|I_{t}|}\right)+\left(\phi(Q_{k}^{(t)})\frac{|\Ical_{_{k}}|}{|I_{t}|}- \frac{1}{|I_{t}|}\sum_{i\in I_{t}}\phi_{k}^S(q_t(h_{i}))\right). 
\end{align}
The second term in the right-hand side  of \eqref{eq:11} is further simplified as 
\begin{align*}
& \phi(Q_{k}^{(t)})\frac{|\Ical_{_{k}}|}{|I_{t}|}- \frac{1}{|I_{t}|}\sum_{i\in I_{t}}\phi_{k}^S(q_t(h_{i})=0,
\end{align*}
since
$
\frac{1}{|I_{t}|}\sum_{i\in I_{t}} \phi_{k}^S(q_t(h_{i}))=\frac{1}{|I_{t}|}  \sum_{i \in \Ical_{k}}\phi(q_t(h_{i})), 
$
and
$
\phi(Q_{k}^{(t)})\frac{|\Ical_{_{k}}|}{|I_{t}|}= \frac{1}{|I_{t}|}\sum_{i \in \Ical_{k}}\phi(q_t(h_{i})).
$
Substituting these into equation \eqref{eq:11} yields
\begin{align} \label{eq:14} 
 \EE_{h}[\phi_{k}^{S}(q_t (h_{}))| \hht(h)= t]-    \frac{1}{|I_{t}|}\sum_{i\in I_{t}}\phi_{k}^S(q_{t}(h_{i})) &=\phi(Q_{k}^{(t)})\left(\Pr(q_t (h_t)=Q_{k}^{(t)}) - \frac{|\Ical_{k}|}{|I_{t}|}\right)
\end{align}
Let $p_k=\Pr(q_t (h_{t})= Q_{k}^{(t)}) $ and $\hat p= \frac{|\Ical_{k}|}{|I_{t}|}$. Consider the random variable $X_i=\one\{q_t (h_{i})= Q_{k}^{(t)}\}$ with the pushforward measure of the random variable $h_i$ under the  map $h_{i}\mapsto\one\{q_t (h_{i})= Q_{k}^{(t)}\} $.
Here, we have that  $X_i \in\{0,1\} \subset [0,1]$. Since  $e$ is fixed and $h_1,\dots,h_n$ are assumed to be iid,  the Hoeffding's inequality
implies the following:  for each fixed $k \in [L^G]$,   
$$
\Pr(p_k -\hat p_k \ge t) \le \exp\left(- 2nt^2 \right).
$$ 
By solving $\delta'= \exp\left(- 2nt^2 \right)$, this implies that for each fixed $k \in [L_{t}^{G_t}]$, for any $\delta'>0$, with probability at least $1-\delta'$,
$$
p_k -\hat p_k \le \sqrt{\frac{\ln(1/\delta')}{2n}}.
$$
By taking union bounds over $k \in [L_{t}^{G_t}]$ with $\delta'=\frac{\delta}{L_t^{G_t}}$, we have that  for any $\delta>0$, with probability at least $1-\delta$,
the following holds  for all $k \in [L_{t}^{G_t}]$:
$$
p_k -\hat p_k\le \sqrt{\frac{\ln(L_{t}^{G_t}/\delta)}{2n}}.
$$
Substituting this into equation \eqref{eq:14} yields that 
for any $\delta>0$, with probability at least $1-\delta$,
the following holds  for all $k \in [L_{t}^{G_{t}}]$:
\begin{align*}
 \EE_{h}[\phi_{k}^{S}(q_t (h_{}))| \hht(h)= t] -     \frac{1}{|I_{t}|}\sum_{i\in I_{t}}\phi_{k}^S(q_{t}(h_{i})) &  \le \phi(Q_{k}^{(t)}) \sqrt{\frac{\ln(L_{t}^{G_{t}}/\delta)}{2n}}\le  \alpha \sqrt{\frac{G_{t}\ln(L_{t})+\ln(1/\delta)}{2n}}.
\end{align*}
Since $t \in \{1,\dots, N\}$ was arbitrary, this holds for an arbitrary $t\in \{1,\dots, N\}$ with probability at least $1-\delta$. By taking union bounds over all  $t\in \{1,\dots, N\}$, we have that for any $\delta>0$, with probability at least $1-\delta$,
the following holds  for all $t \in [N]$ and  $k \in [L_{t}^{G_{t}}]$:
 \begin{align} \label{eq:17}
 \EE_{h}[\phi_{k}^{S}(q_t (h_{}))| \hht(h)= t] -     \frac{1}{|I_{t}|}\sum_{i\in I_{t}}\phi_{k}^S(q_{t}(h_{i})) &  \le\alpha \sqrt{\frac{G_{t}\ln(L_{t})+\ln(N/\delta)}{2n}}.
\end{align}
Combining \eqref{eq:15}, \eqref{eq:16}, and \eqref{eq:17} with union bounds, we have that for any $\delta>0$, with probability at least $1-\delta$,
the following holds for all  $k \in [L_{t}^{G_{t}}]$:
\begin{align*}
\EE_{h}[\phi_{k}^S(q(h))] - \frac{1}{n}\sum_{i=1}^n\phi_{k}^S(q(h_{i})) 
&\le\alpha\one\{N\ge 2\} \sqrt{\frac{2N \ln2 + 2 \ln(1/\delta)}{n}}+\sum_{t\in \Tcal_{k}} \frac{|I_{t}|}{n} \left( \alpha \sqrt{\frac{G_{t}\ln(L_{t})+\ln(N/\delta)}{2n}} \right) 
\\ & =\alpha\left(\sum_{t\in \Tcal_{k}} \frac{|I_{t}|}{n}  \sqrt{\frac{G_{t}\ln(L_{t})+\ln(N/\delta)}{2n}}+\one\{N\ge 2\}\sqrt{\frac{2N \ln2 + 2 \ln(1/\delta)}{n}}\right) 
\end{align*}
Since $\Tcal_{k} \subseteq \{1,\dots,N\}$, this proves the desired statement of this theorem.

\end{proof}

\begin{figure*}
    \centering
    \includegraphics[width=0.45\linewidth]{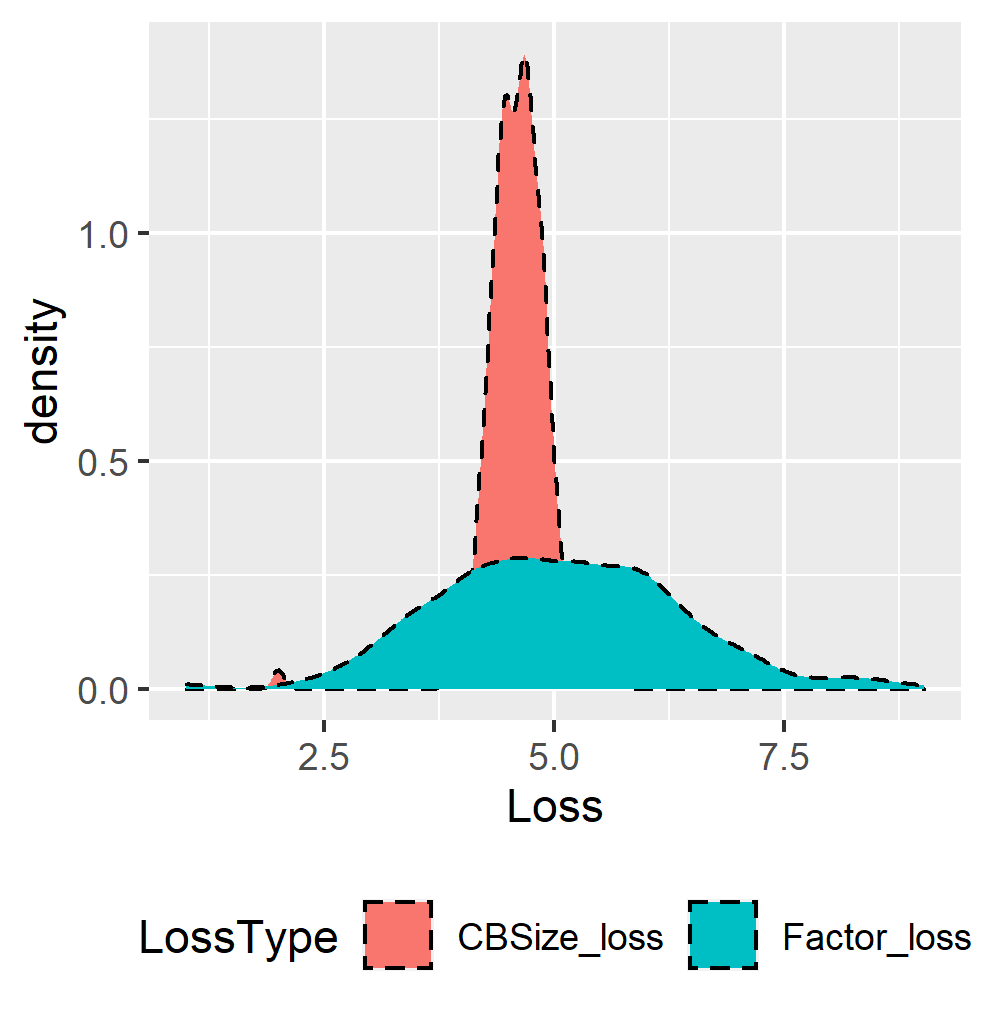}
    \includegraphics[width=0.45\linewidth]{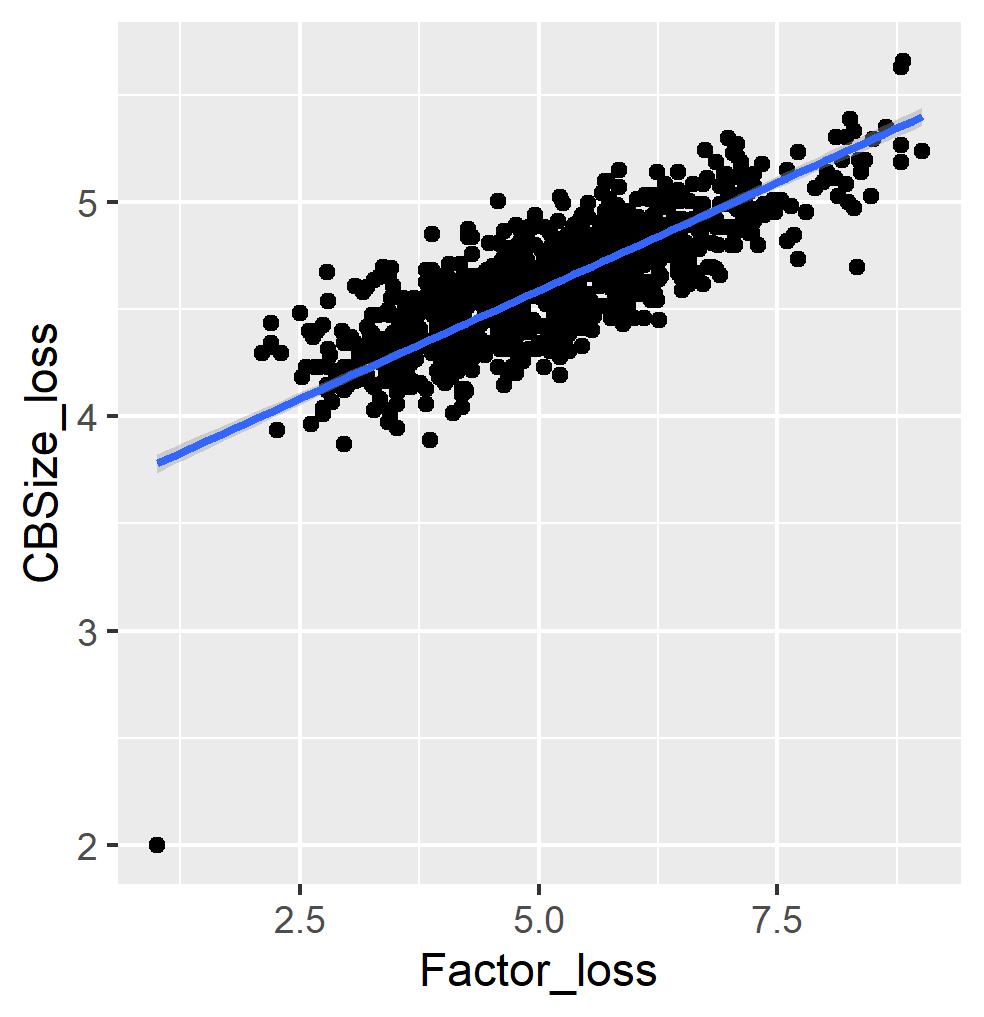}
    \caption{Upper panel: Disbution of factor loss and codebook size loss. Lower panel: Correlation between factor loss and codebook loss in DVQ}
    \label{fig:FactorCBCor}
\end{figure*}


\end{document}